\documentclass{article}
\usepackage{spconf,graphicx,hyperref}
\usepackage{amsmath}
\usepackage{amssymb}
\usepackage{mathtools}
\usepackage{amsthm}
\usepackage{bm}
\usepackage{dsfont}
\usepackage{xcolor}
\usepackage{cleveref}

\usepackage{mdframed}
\usepackage{amsthm}
\newtheorem{theorem}{Theorem}
\newtheorem{definition}{Definition}
\newtheorem{assumption}{Assumption}

\newtheorem{remark}{Remark}

\newtheorem{corollary}{Corollary}

\usepackage{afterpage}

\usepackage[textsize=tiny]{todonotes}

\usetikzlibrary{arrows,shapes,calc,decorations.pathreplacing,intersections,quotes,angles,positioning,fit,petri,through}
\usepackage{tikz,pgfplots}
    \pgfdeclarelayer{background}
    \pgfdeclarelayer{foreground}
    \pgfsetlayers{background,main,foreground}
\usetikzlibrary{spy}
\usetikzlibrary{shapes}
\usetikzlibrary{plotmarks}

\pgfmathdeclarefunction{gauss}{2}{%
   \pgfmathparse{1/(#2*sqrt(2*pi))*exp(-((x-#1)^2)/(2*#2^2))}%
}

\usepackage{amsmath,amsfonts,bm}

\def\eqref#1{equation~\ref{#1}}

\def\1{\bm{1}}

\def\mA{{\bm{A}}}

\DeclareMathAlphabet{\mathsfit}{\encodingdefault}{\sfdefault}{m}{sl}
\SetMathAlphabet{\mathsfit}{bold}{\encodingdefault}{\sfdefault}{bx}{n}

\title{A Random Matrix Perspective of Echo State Networks: From Precise Bias--Variance Characterization to Optimal Regularization}

\name{
  Yessin Moahker $^{ \ddagger}$, 
  Malik Tiomoko $^{ \ddagger}$, 
  Cosme Louart $^{\dagger}$, 
  Zhenyu Liao $^{\mathsection}$
}

\address{
  $^{\ddagger}$ Huawei Noah’s Ark Lab, Huawei Technologies, Paris, France\\
  $^{\dagger}$ Chinese university of Hongkong, Shenzhen, China\\
  $^{\mathsection}$ Huazhong University of Science Technology (HUST), Wuhan, China
}
\pgfplotsset{compat=newest} 
\begin{document}
%
\maketitle

\begin{abstract}
We present a rigorous asymptotic analysis of Echo State Networks (ESNs) in a teacher--student setup.
Leveraging techniques from random matrix theory, we derive closed-form expressions for the asymptotic bias, variance, and mean-squared error (MSE) as functions of the input statistics, the oracle vector, and the regularization parameter. 
The analysis reveals two key departures from classical ridge regression: 
(i) ESNs do not exhibit double descent, and 
(ii) ESNs attain lower MSE  when both the number of training samples and the teacher's memory length are limited. 
Based on these analytic results, we further derive an explicit formula of the optimal regularization for isotropic inputs.
Together, these results offer interpretable theory and practical guidelines for tuning ESNs, helping reconcile recent empirical observations with provable performance guarantees.
\end{abstract}

\section{Introduction}

Echo State Networks (ESNs) are a class of recurrent neural networks (RNNs) widely recognized for their computational efficiency and strong empirical performance in modeling temporal data \cite{jaeger2002adaptive,sun2022systematic,ferte2024reservoir}. 
They have been successfully applied to tasks including time series forecasting \cite{ferte2024reservoir}, control \cite{salmen2005echo},  speech processing \cite{skowronski2007automatic}, and computational neuroscience \cite{bozhkov2016learning}. 
Their appeal stems from a simple training scheme: only the output layer is learned, while the recurrent weights are fixed at initialization.
Despite this simplicity, ESNs can capture complex dynamical dependencies and memory effects, often empirically outperforming fully trained RNNs, particularly in limited-data regimes~\cite{williams2024reservoir}.

Despite widespread adoption, our theoretical understanding of ESNs remains limited. 
Recent works have begun to address this gap using tools from statistical physics and random matrix theory (RMT), to analyze generalization of ESNs in high dimensions \cite{couillet2016training}, ensemble methods to reduce instability \cite{wu2018statistical}, and dimensionality reduction techniques to improve robustness \cite{lokse2017training}. 
However, a unified framework explaining ESN generalization in a realistic, temporally dependent setting is still lacking. 
Also, hyperparameter tuning of ESNs—especially regularization—relies on costly heuristic strategies such as grid search, Bayesian optimization \cite{maat2018efficient}, or evolutionary methods \cite{matzner2022hyperparameter}, with minimal theoretical guidance \cite{soltani2023echo}.

In parallel, the machine learning (ML) community has leveraged RMT to study high-dimensional ML in static settings. 
This has yielded insights into linear ML models such as ridge regression, uncovering intriguing phenomena ranging from double descent \cite{liao2020random}, scaling laws \cite{atanasov2024scaling}, to the impact of data covariance \cite{ba2023learning}.
A few studies have applied RMT-type analyses to ESNs, but often under restrictive assumptions. 
For example, \cite{couillet2016training} considers a Gaussian model with deterministic targets and a single training sample, producing precise but idealized asymptotic errors, while \cite{grigoryeva2018echo} investigates ESN capacity for fading-memory processes without providing exact performance predictions.  

In contrast, here we adopt a teacher--student framework with stochastic targets, multiple time series, and general input distributions. 
This allows us to compare the performance of ESNs against that of ridge regression, to study optimal regularization, and to provide realistic predictions of generalization performance, including scenarios where ESNs mitigate double descent and exploit temporal structure efficiently.

\paragraph*{Contributions.} 
In this paper, we present a rigorous and asymptotically exact theory for ESN generalization:
\begin{enumerate}
    \item We provide precise characterizations of ESN bias and variance as functions of input statistics, teacher model properties, and regularization strength.
    \item We demonstrate that ESNs can mitigate double descent and outperform ridge regression in limited-data and short-memory regimes, and we derive closed-form optimal regularization for isotropic inputs.
\end{enumerate}

\paragraph*{Notation.} 
Bold uppercase letters (e.g., $\mA$) denote matrices, bold lowercase (e.g., $\mathbf{x}$) vectors, and plain lowercase (e.g., $x$) scalars. Norms are $\|\mathbf{x}\|_2$, $\|\mA\|$, and $\|\mA\|_F$. For sequences $u,v$, $u=O(v)$ denotes asymptotic boundedness with high probability. Expectation is denoted $\mathbb{E}[\cdot]$.


\section{Problem Setup and Assumptions}
\label{sec:setup}

We consider a supervised learning task under a teacher--student framework.  
The input is a temporal signal $\mathbf{u} \in \mathbb{R}^T$ of length $T$, and the target output $y \in \mathbb{R}$ is generated by the following noisy linear teacher model.

\begin{definition}[\textbf{Noisy Linear Teacher}]
\label{def:teacher}
The input--output pair $(\mathbf{u}, y) \in \mathbb{R}^T \times \mathbb{R}$ is drawn the following noise linear model
\begin{equation}
\label{eq:teacher_model}
y = \boldsymbol{\theta}_\ast^\top \mathbf{u} + \epsilon,
\end{equation}
where $\boldsymbol{\theta}_\ast \in \mathbb{R}^{T}$ is the ground--truth parameter and $\epsilon$ is zero--mean noise with variance $\sigma^2$, independent of $\mathbf{u}$.
\end{definition}
\noindent
To extract structural information from the input signal $\mathbf{u}$,  the student model first applies a \emph{fixed} transformation $F \colon \mathbb{R}^T \to \mathbb{R}^n$ to obtain some feature $\mathbf{z} = F(\mathbf{u})$ of $\mathbf{u}$, as follows.

\begin{definition}[\textbf{Feature Map}]\label{def:feature_map}
We consider the following two types of feature maps $F \colon \mathbb{R}^T \to \mathbb{R}^n$:
\begin{enumerate}
    \item \textbf{Ridge regression}: $\mathbf{z} = F(\mathbf{u}) = \mathbf{u}$, i.e., the obtained features are the raw inputs.
    \item \textbf{Linear ESN}: the feature map $\mathbf{z} = F(\mathbf{u})$ is obtained from a fixed recurrent dynamical system driven by $\mathbf{u}$.  
    Specifically, the ESN computes a state $\mathbf{x}_t \in \mathbb{R}^n$ updated as
    \begin{equation}
    \label{eq:recurrence}
    \mathbf{x}_{t+1} = \mathbf{W}\,\mathbf{x}_t + \mathbf{w}_{\mathrm{in}}\,u_t, \quad \mathbf{u} = [u_1, \ldots, u_T]^\top,
    \end{equation}
    where $\mathbf{W} \in \mathbb{R}^{n\times n}$ (recurrent weights) and $\mathbf{w}_{\mathrm{in}} \in \mathbb{R}^n$ (input weights) are  fixed at random, and $u_t$ is the $t^{th}$ input sample.  
    The corresponding feature is obtained by taking the last state as
    \begin{equation}
        \mathbf{z} = F(\mathbf{u}) = \mathbf{x}_T.
    \end{equation}
\end{enumerate}
\end{definition}
\noindent
For the sake of presentation, we consider here in \Cref{def:feature_map} two specific types of feature maps $F(\cdot)$. 
Our results in \Cref{theo:main} hold much more generally for \emph{any} fixed feature maps.

\paragraph*{Training by Ridge Regression.}
The student model then approximates the teacher (in \Cref{def:teacher}) by learning from $N$ i.i.d.\@ training pairs $\{(\mathbf{u}_i,y_i)\}_{i=1}^N$.  
As discussed in \Cref{def:feature_map}, each input $\mathbf{u}_i$ is first mapped through some fixed feature map $\mathbf{z}_i = F(\mathbf{u}_i)$, and collected as the columns of $\mathbf{Z} = [\mathbf{z}_1, \ldots, \mathbf{z}_N] \in \mathbb{R}^{n \times N}$.  
The readout weights $\mathbf{w}_{\mathrm{out}} \in \mathbb{R}^{n}$ are then obtained using the following ridge regression:
\begin{equation}\label{eq:def_ridge_reg}
\hat{\mathbf{w}}_{\mathrm{out}} := \arg\min_{\mathbf{w}} \frac{1}{N} \left\| \mathbf{y} - \mathbf{Z}^\top \mathbf{w} \right\|_2^2 + \lambda\|\mathbf{w}\|_2^2,
\end{equation}
with $\mathbf{y} = [y_1, \ldots, y_N]^\top \in \mathbb{R}^N$ and regularization parameter $\lambda > 0$. 
The solution to \eqref{eq:def_ridge_reg} is explicitly given by
\begin{equation}\label{eq:def_W_out}
    \hat{\mathbf{w}}_{\mathrm{out}} = \left( \frac{1}{N} \mathbf{Z}\mathbf{Z}^\top + \lambda \mathbf{I}_n \right)^{-1} \frac{1}{N} \mathbf{Z} \mathbf{y}.
\end{equation}
The teacher--student under study is summarized in \Cref{fig:setup}.

At test time, a fresh input $\mathbf{u}'$ is mapped to $\mathbf{z}'=F(\mathbf{u}')$, and the student model predicts as $\hat{y}' = \hat{\mathbf{w}}_{\mathrm{out}}^\top \mathbf{z}'$.


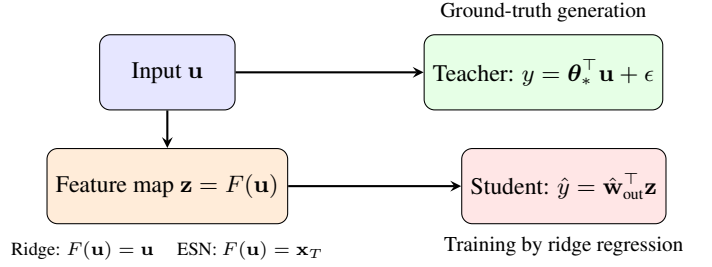
\begin{figure}[t]
\centering
\begin{tikzpicture}[>=stealth, node distance=2.0cm, every node/.style={font=\small}]
    \node[draw, rounded corners, minimum width=1.8cm, minimum height=1cm, fill=blue!10] (input) {Input $\mathbf{u}$};

    \node[draw, rounded corners, right=2.5cm of input, minimum width=2.5cm, minimum height=1cm, fill=green!10] (teacher) {Teacher: $y=\boldsymbol{\theta}_\ast^\top\mathbf{u}+\epsilon$};

    \node[draw, rounded corners, below=0.5cm of input, minimum width=2.5cm, minimum height=1cm, fill=orange!15] (feature) {Feature map $\mathbf{z} = F(\mathbf{u})$};
    \node[below=0.1cm of feature] {\scriptsize Ridge: $F(\mathbf{u})=\mathbf{u}$ \quad ESN: $F(\mathbf{u})=\mathbf{x}_{T}$};

    \node[draw, rounded corners, right=2.3cm of feature, minimum width=2.5cm, minimum height=1cm, fill=red!10] (student) {Student: $\hat{y}=\hat{\mathbf{w}}_{\text{out}}^\top \mathbf{z}$};

    \draw[->, thick] (input) -- (teacher);
    \draw[->, thick] (input) -- (feature);
    \draw[->, thick] (feature) -- (student);

    \node[above of=teacher, node distance=0.8cm, font=\footnotesize] {Ground-truth generation};
    \node[below of=student, node distance=0.8cm, font=\footnotesize] {Training by ridge regression};
\end{tikzpicture}
\caption{{The teacher--student framework considered in this paper. 
The teacher generates $y$ from the inputs $\mathbf{u}$ through some ground--truth parameter $\boldsymbol{\theta}_\ast $, see \Cref{def:teacher}. 
The student trains a ridge regression readout on the features $F(\mathbf{u})$ in \Cref{def:feature_map}, obtained either from raw inputs ($F(\mathbf{u})=\mathbf{u}$) or ESN final states ($F(\mathbf{u})=\mathbf{x}_T$) obtained recurrently from \eqref{eq:recurrence}.}}
\label{fig:setup}
\end{figure}


To evaluate the generalization performance of the student model in \Cref{fig:setup}, we focus on the asymptotic behavior of its out-of-sample risk—the expected mean squared error (MSE) on the fresh test pair $(\mathbf{u}', y')$.
This is defined as follow.
\begin{definition}[\textbf{Out-of-Sample Risk}]\label{def:outofsample}
For an independent test pair $(\mathbf{u}', y')$ drawn from the teacher model in \Cref{def:teacher}, the out-of-sample risk of a student model with readout vector $\hat{\mathbf{w}}_{\mathrm{out}}$ in \eqref{eq:def_W_out} is given by
\[
\mathcal{R} := \mathbb{E} \big[ \| \hat{y}' - y' \|_2^2 \big], \quad \hat{y}' = \hat{\mathbf{w}}_{\mathrm{out}}^\top F(\mathbf{u}'),
\]
where the expectation is taken over the training set, the test input $\mathbf{u}'$, and the noise $\epsilon$ in \eqref{eq:teacher_model}. 
\end{definition}

We work under the following assumptions.
\begin{assumption}[\textbf{High-dimensional Asymptotics}]
\label{assum:asymptotic}
The feature dimension $n$ and the sample size $N$ are both large and comparable, that is, $n/N \to \gamma \in (0,\infty)$ as $n, N \to \infty$.
\end{assumption}

\begin{assumption}[\textbf{Concentration of feature vectors}]
\label{assum:concentration}
The feature vectors $\mathbf{z} \in \mathbb{R}^n$ are \emph{$q$-exponentially concentrated}, i.e., for any $1$-Lipschitz function 
$\varphi : \mathbb{R}^n \to \mathbb{R}$,
\[
\mathbb{P}\big(|\varphi(\mathbf{z}) - \mathbb{E}[\varphi(\mathbf{z})]| \ge t\big)
\le C e^{-(t/\sigma)^q}, \quad \forall t>0,
\]
for some $q>0$, $C>0$, and $\sigma>0$ independent of $n$. 
\end{assumption}
\noindent
The class of concentrated random vectors in \Cref{assum:concentration} includes multivariate Gaussian vectors, vectors uniformly distributed on the sphere, and any Lipschitz transform thereof (e.g., almost realistic images generated by GANs~\cite{seddik2020random}).  
Intuitively, smooth observations of such $\mathbf{z}$ tightly concentrate around their means, with fluctuations of order $O(1)$.

\section{Main Theoretical Results}
\label{sec:main_results}

In this section, we provide asymptotically exact bias–variance formulas for the out-of-sample risk of student model with \emph{any} fixed feature map $F$, and specialize them to linear ESNs and ridge regression. 
As an important consequence of this characterization, we derive \emph{optimal} regularization for linear ESNs.


Before presenting our main theoretical result, let us introduce the following notations.
Define $\delta$ as the unique positive solution to the following fixed-point equation
\begin{equation}\label{eq:def_delta}
    \delta = \frac{1}{N} \mathrm{Tr}(\mathbf{\Sigma}_z \bar{\mathbf{Q}}), \quad \bar{\mathbf{Q}} = \left(\frac{\mathbf{\Sigma}_z}{1+\delta} + \lambda \mathbf{I}_n\right)^{-1},
\end{equation}
and
\begin{equation}\label{eq:def_alpha}
    \alpha = \frac{1}{N} \left\|
    \frac{\mathbf{\Sigma}_z}{1+\delta} \bar{\mathbf{Q}} \right\|_F^2, \quad \mathbf{\Sigma}_z = \mathbb{E}[\mathbf{z}\mathbf{z}^\top] \in \mathbb{R}^{n \times n},
\end{equation}
as well as
\begin{equation}\label{eq:def_Sigmas}
\mathbf{\Sigma}_u = \mathbb{E}[\mathbf{u}\mathbf{u}^\top] \in \mathbb{R}^{T \times T}, \quad
\mathbf{\Sigma}_{uz} = \mathbb{E}[\mathbf{u}\mathbf{z}^\top] \in \mathbb{R}^{T \times n}.
\end{equation}

\begin{remark}
In the case of ridge regression and linear ESN in \Cref{def:feature_map}, the expressions of $\mathbf{\Sigma}_z, \mathbf{\Sigma}_u, \mathbf{\Sigma}_{uz}$ in \eqref{eq:def_Sigmas} take the following closed forms:
\begin{itemize}
    \item \textbf{Ridge regression} where $\mathbf{z} = F(\mathbf{u}) = \mathbf{u}$, we have $\mathbf{\Sigma}_z = \mathbf{\Sigma}_u = \mathbf{\Sigma}_{uz} $.
    \item \textbf{Linear ESN} as in \Cref{def:feature_map}, let
\[
\mathbf{S} = [\,\mathbf{W}^{T-1}\mathbf{w}_{\mathrm{in}}, \dots, \mathbf{W} \mathbf{w}_{\mathrm{in}}, \mathbf{w}_{\mathrm{in}}\,] \in \mathbb{R}^{n\times T}.
\]
Then we have 
\[
\mathbf{\Sigma}_{uz} = \mathbb{E}[\mathbf{\Sigma}_u \mathbf{S}^\top], \quad
\mathbf{\Sigma}_z = \mathbb{E}[\mathbf{S} \mathbf{\Sigma}_u \mathbf{S}^\top].
\]
\end{itemize}
For generic \emph{nonlinear} $F$, these quantities remain well-defined but involve computations of high-dimensional integrations.
\end{remark}

With these notations at hand, we are ready to present our main technical result that precisely characterizes the bias and variance of generic student model.
\begin{theorem}[\textbf{Precise Bias--Variance Characterization}]\label{theo:main}
Let Assumptions~\ref{assum:asymptotic}~and~\ref{assum:concentration} hold, then, the out-of-sample risk $\mathcal{R}$ in \Cref{def:outofsample} of a student model with feature map $F(\cdot)$ satisfies 
\[
\mathcal{R} = \mathcal{B}^2 + \mathcal{V} + \sigma^2,
\]
with $\mathcal{B}^2 - \frac{1}{1-\alpha} \Big[
\boldsymbol{\theta}_\ast^{\top} \mathbf{\Sigma}_u \boldsymbol{\theta}_\ast
- \frac{2 \boldsymbol{\theta}_\ast^{\top} \mathbf{\Sigma}_{uz} \bar{\mathbf{Q}} \mathbf{\Sigma}_{uz}^\top \boldsymbol{\theta}_\ast}{1+\delta} 
+ 
\frac{\boldsymbol{\theta}_\ast^{\top} \mathbf{\Sigma}_{uz} \bar{\mathbf{Q}} \mathbf{\Sigma}_z  \bar{\mathbf{Q}} \mathbf{\Sigma}_{uz}^\top \boldsymbol{\theta}_\ast}{(1+\delta)^2}
\Big] \\ \to 0 $ and $\mathcal{V} - \frac{\sigma^2 \alpha}{1-\alpha} \to 0$ as $n,N \to \infty$, where $\sigma^2$ is the noise variance in \Cref{def:teacher}.
\end{theorem}
\begin{proof}[Proof sketch of \Cref{theo:main}]  
To establish the precise bias-variance characterizations in \Cref{theo:main}, we decompose the risk \(\mathcal{R}\) into bias, variance, and noise components. 
The primary technical challenge here lies in analyzing the asymptotic behavior of the bias term \(\mathcal{B}^2\).
Specifically, we employ RMT techniques such as the concentration of measure~\cite{louart2018concentration} and deterministic equivalents~\cite[Chapter 6]{couillet2011random} to simplify the expectations involving the random feature states of ESNs. 
Importantly, \Cref{theo:main} extends previous efforts on the bias-variance trade-off for classical regression models~\cite{bach2024high}, by incorporating the effects of reservoir dynamics in ESNs, thereby offering a more nuanced understanding of their generalization performance compared to classical models.
\end{proof}


As a consequence of \Cref{theo:main}, we have the following results for linear ESNs.
\begin{corollary}[\textbf{Out-of-Sample Risk of Linear ESNs}]\label{coro:risk_linear_ESN}
Recall from \Cref{def:feature_map} the feature map of linear ESNs as $\mathbf{z} = F(\mathbf{u}) = \mathbf{x}_T = \sum_{t=0}^{T-1} \mathbf{W}^{t} \mathbf{w}_{\mathrm{in}}\, u_{T - t}$.
Then, under the settings and notations of \Cref{theo:main}, the asymptotic bias and variance of linear ESNs are explicitly given by
\[
\textstyle \mathcal{B}^2 - \frac{(1+\delta)^2}{1-\alpha} \sum_t \alpha_t (\boldsymbol{\theta}_{\ast}^\top \mathbf{\Sigma}_u^{1/2} \mathbf{v}_t)^2 \to 0,~
\mathcal{V} - \frac{\sigma^2}{N(1-\alpha)} \sum_t \beta_t \to 0,
\]
where \( \{ (\mu_t, \mathbf{v}_t) \}_{t=1}^T\) are the eigenvalue-eigenvector pairs of 
\[
\mathbf{\Sigma}_u^{1/2} \, \mathrm{diag}(\varphi^{t - T})_{t=1}^T \, \mathbf{\Sigma}_u^{1/2},
\]
and \(\varphi \in (0,1]\) is the \emph{leak factor}, i.e., the leaking rate parameter that controls the trade-off between memory retention and update speed~\cite{jaeger2007optimization}, and
\[
\alpha_t = \frac{\lambda^2}{\left(\mu_t + \lambda(1+\delta)\right)^2},
\quad
\beta_t = \frac{\mu_t}{ \left(\mu_t + \lambda(1+\delta)\right)^2}.
\]

\end{corollary}
%
%
%

Exploiting the explicit bias and variance expressions in \Cref{coro:risk_linear_ESN}, we derive, in the following result, the optimal regularization for linear ESNs with isotropic inputs.

\begin{corollary}[\textbf{Optimal Regularization for Linear ESNs}]\label{coro:optimal-regu}
For isotropic inputs with \( \boldsymbol{\Sigma}_u = \mathbf{I}_T\),  the optimal regularization that minimizes the (asymptotic) out-of-sample risk $\mathcal{R}$ in \Cref{def:outofsample} is proportional to the signal-to-noise ratio (SNR) and is given by
\[
\lambda_\star = \frac{T}{N} \cdot \mathrm{SNR}, \quad \text{where} \quad \mathrm{SNR} = \frac{\| \boldsymbol{\theta}_{\ast} \|_2^2}{\sigma^2}.
\]
\end{corollary}

\section{Insights and Experiments}
\label{sec:insights_experiments}

In this section, we discuss implications of our main theoretical findings and provide experimental evidence. In our experiments, we model short-term memory by setting the oracle vector as \( \boldsymbol{\theta}_\star = \{  \rho^{-t} \}_{t=0}^{T-1} \), so that smaller \(\rho\) emphasizes recent inputs (short memory), while \(\rho \approx 1\) yields longer memory behavior.

\subsection{Why ESNs Do Not Exhibit Double Descent?}
The well-known \emph{double descent} phenomenon in ridge regression arises when the effective model complexity, captured from~\eqref{eq:def_alpha} by 
\[
\alpha := \frac{1}{N} \left\|
    \frac{\mathbf{\Sigma}_z}{1+\delta} \bar{\mathbf{Q}} \right\|_F^2 
    = \frac{1}{N} \sum_{t=1}^T \frac{\mu_i^2}{\bigl(\mu_t + \lambda(1+\delta)\bigr)^2},
\] 
approaches~1 (c.f.\@ the denominator of \(1-\alpha\) for both bias and variance in \Cref{coro:risk_linear_ESN}), which causes a spiking peak in bias and variance at the interpolation threshold (\(T \approx N\) and as \(\lambda \to 0\)).  
For linear ESNs, however, the eigenvalues \(\{\mu_i\}\) arise from  $ 
\mathbf{\Sigma}_u^{1/2}\, \mathrm{diag}(\varphi^{t-T})_{t=1}^T \mathbf{\Sigma}_u^{1/2},$  
where the leak factor \(\varphi \in (0,1]\) enforces an approximately low-rank structure and keeps \(\alpha\) bounded strictly below~1.  
Hence, ESNs naturally avoid double descent, except in degenerate regimes with no leakage (\(\varphi=1\)) and vanishing regularization.




\begin{figure}[t]
    \centering
    \includegraphics[width=\linewidth, height=0.5\linewidth]{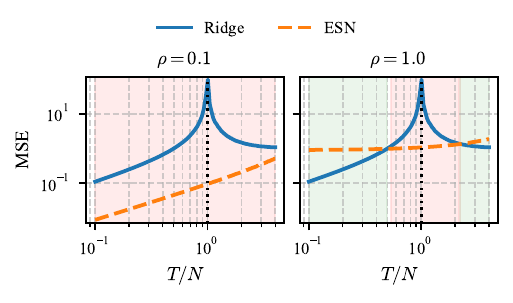}
    \vspace{-1cm}
    \caption{
Comparison of test error versus context length for ESN and ridge regression. Ridge exhibits double descent at \(\gamma \approx 1\), while ESNs remain smooth due to limited memory. 
\textbf{Left}: short memory (\(\rho \ll 1\)); \textbf{right}: long memory (\(\rho \approx 1\)).
    }
    \label{fig:double_descent}
    \vspace{-0.2cm}
\end{figure}

\Cref{fig:double_descent} illustrates this behavior by comparing the out-of-sample test errors of ridge regression and ESNs as a function of the dimension ratio \( T/N \).

\subsection{When ESNs Outperform Ridge Regression?}

Our theoretical results show that ESNs can outperform ridge regression in the \emph{limited data, short-memory} regime. 
This advantage stems from the ESN's \emph{temporal inductive bias}, which assumes that task-relevant features lie in the \emph{recent past}.  

In contrast, ridge regression on raw inputs imposes no temporal structure and may overfit when data is scarce. 
ESNs, by \emph{filtering and weighting historical inputs}, more effectively extract temporal patterns. 
For short-memory tasks, a leak factor $\varphi < 1$ downweights distant past inputs, producing an approximately low-rank eigenvalue spectrum (with $\alpha < 1$). 
This mitigates the spiking peaks in test errors, focusing learning on recent, informative inputs.

Consequently, when $N \ll n = T$, ridge regression exhibits high variance, while the ESN’s recurrent structure efficiently captures short-term dependencies, as reflected in the MSE curves in \Cref{fig:short_memory}.

\begin{figure}[t]
    \centering
\includegraphics[width=\columnwidth, height=0.45\linewidth]{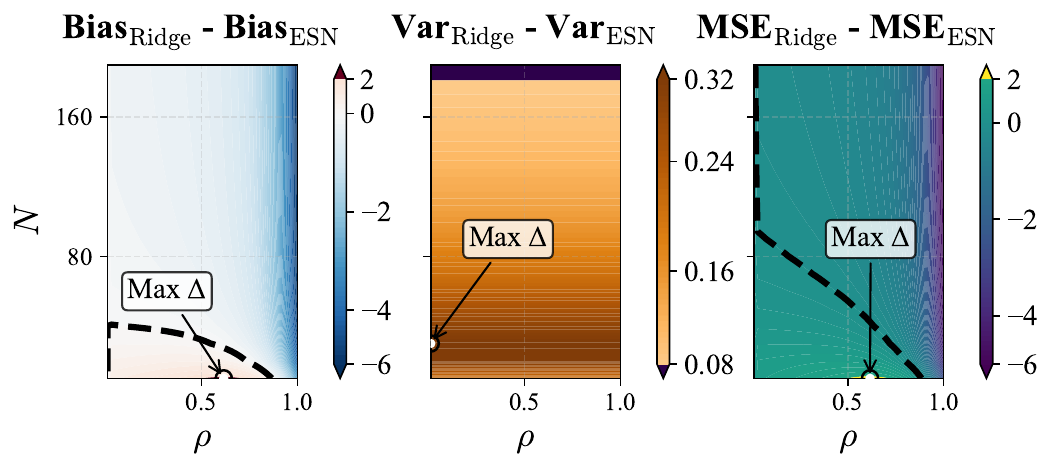}
    \caption{
    Differences in bias, variance, and MSE between ridge regression and ESNs are shown as a function of training size \(N\) and memory parameter \(\rho\), where small \(\rho\) emphasizes recent inputs. ESNs outperform ridge regression when \(N\) is small and the task relies on recent inputs.
    }
    \label{fig:short_memory}
\end{figure}



\subsection{Optimal Regularization for Linear ESNs}

Finally, we validate our theoretical predictions for the \emph{optimal regularization} \( \lambda_\star = \frac{T}{N} \cdot \mathrm{SNR} \) given in \Cref{coro:optimal-regu} for isotropic inputs.

As illustrated in \Cref{fig:optimal_lambda}, the test error is minimized near the theoretically predicted \( \lambda_\star \), supporting the practical relevance of our asymptotic analysis.

\begin{figure}[t]
    \centering
    \includegraphics[width=\columnwidth, height=0.45\linewidth]{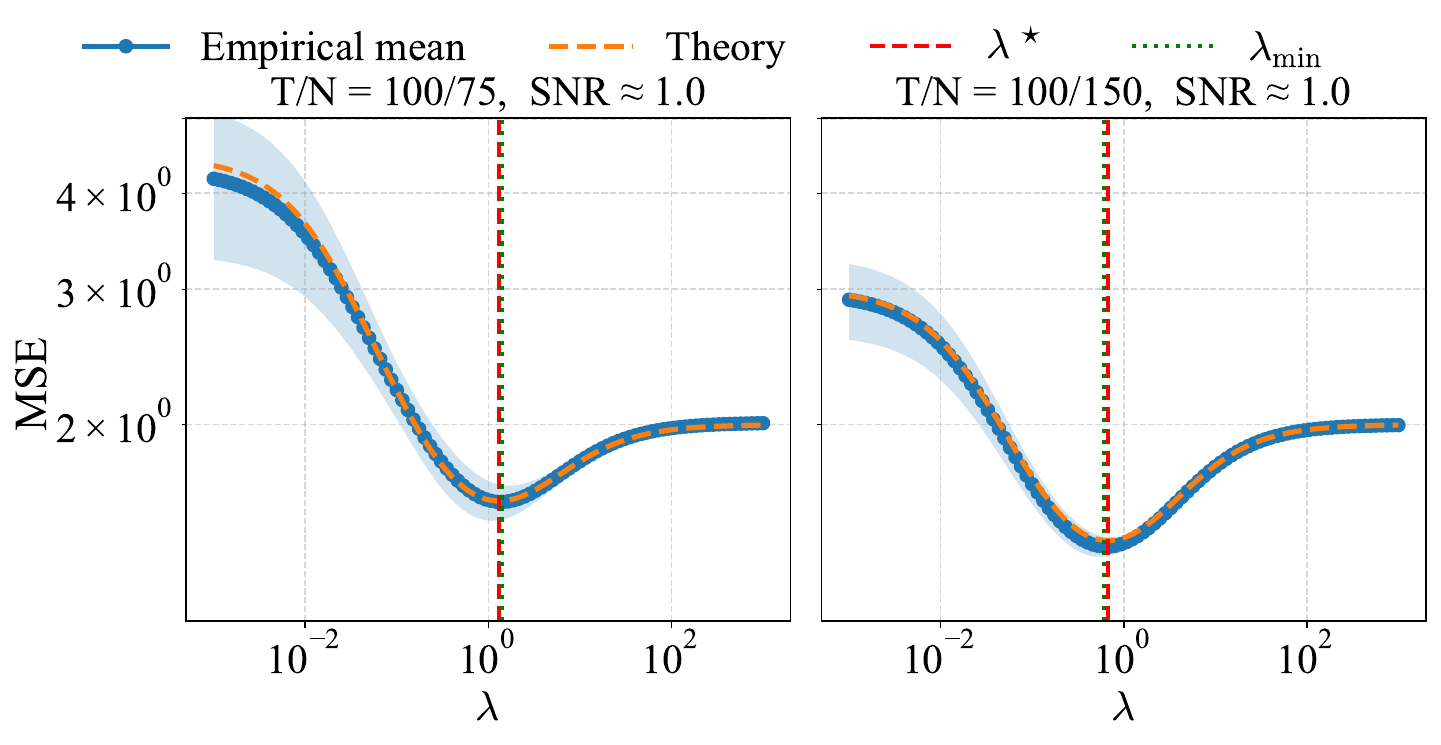}
    \caption{
    Test error as a function of regularization \( \lambda \). 
    The theoretically optimal \( \lambda_\star \) in \Cref{coro:optimal-regu} (vertical dashed line) closely matches the empirically optimal values $\lambda_{\min}$.
    }
    \label{fig:optimal_lambda}
    \vspace{-0.2cm}
\end{figure}

\section{Concluding remarks}
\label{sec:conclusion}

This work analyzes Echo State Networks (ESNs) using random matrix theory to derive closed-form bias and variance formulas. 
The proposed analysis sheds novel insights on ESNs' robustness to overfitting and absence of double descent due to their limited memory.  

We demonstrate that ESNs outperform (unstructured) ridge regression in data-scarce settings and when recent inputs matter most. 
We also provide closed-form formulas for ESN optimal regularization. 
Future work includes studying estimation errors and extending the theory to nonlinear ESNs and other recurrent models (RNNs, LSTMs).




\hfill
\clearpage
\newpage
\bibliographystyle{IEEEbib}
\bibliography{strings,refs}

\end{document}